\documentclass{article}

\PassOptionsToPackage{dvipsnames,table}{xcolor}

\usepackage{iclr2027_conference,times}

\usepackage{amsmath,amsfonts,bm}

\def\eqref#1{equation~\ref{#1}}

\def\1{\bm{1}}

\DeclareMathAlphabet{\mathsfit}{\encodingdefault}{\sfdefault}{m}{sl}
\SetMathAlphabet{\mathsfit}{bold}{\encodingdefault}{\sfdefault}{bx}{n}

\usepackage[utf8]{inputenc}
\usepackage[T1]{fontenc}

\usepackage{amssymb}
\usepackage{mathtools}
\usepackage{amsthm}

\usepackage{graphicx}
\usepackage{booktabs}
\usepackage{multirow}
\usepackage{subcaption}
\usepackage{wrapfig}

\usepackage{nicefrac}
\usepackage{microtype}
\usepackage{csquotes}
\usepackage{enumitem}
\usepackage{xspace,lineno}

\usepackage{algorithm}
\usepackage{algorithmic,lipsum}

\usepackage[textsize=tiny]{todonotes}

\usepackage{xcolor}
\usepackage[
  colorlinks=true,
  linkcolor=purple,
  citecolor=blue,
  urlcolor=BrickRed
]{hyperref}
\usepackage{cleveref}

\newtheorem{theorem}{Theorem}

\theoremstyle{definition}

\theoremstyle{remark}

\newif\ifshowrevisions
\showrevisionstrue    

\newcommand{\std}[1]{{\,\scriptsize$\pm$#1}}

\title{Heat Field Signatures:\\ From Point Clouds to Smooth Geometry}

\author{Yuanqing Wang \\ Dept. of Math. Sciences \\ University of Texas at Dallas \\ Richardson, TX, USA \\ \And 
Yapeng Tian \\ Dept. of Computer Science \\ University of Texas at Dallas \\ Richardson, TX, USA \\ \And 
Baris Coskunuzer \\ Dept. of Math. Sciences \\ University of Texas at Dallas \\ Richardson, TX, USA }

\iclrfinalcopy 

\begin{document}

\maketitle
\vspace{-.2in}

\begin{abstract}
Bringing multiscale geometric analysis directly to irregular point clouds
remains difficult: quantities such as local dimension, anisotropy, density
variation, and geometric transitions are typically estimated through explicit
neighborhood, manifold, or graph constructions, or left for neural networks to
infer from coordinates. We introduce \emph{Heat Field Signatures} (HFS), which
lift a point cloud to a multiscale family of smooth ambient heat fields,
providing a direct interface from discrete samples to geometric analysis.

From this field, HFS computes closed-form global and local signatures directly
from pairwise distances, capturing heat concentration, intrinsic dimension,
anisotropy, and scale transitions. We further introduce the \emph{Heat
Dimension Spectrum} (HDS), a compact summary of multiscale geometric
composition. HFS can be used as a closed-form descriptor, a lightweight learned
representation, or a geometric feature channel for neural point-cloud models.

Across synthetic and real-world benchmarks spanning subcellular, neuronal,
tree, and protein data, HFS outperforms strong point-cloud and
multiparameter-persistence baselines while substantially reducing end-to-end
cost. On SCOP protein-fold classification, HFS improves over the strongest
deep baseline by nearly $24$ percentage points using coordinates alone, while
standalone HFS representations are exactly rotation-invariant by construction.
More broadly, HFS turns a classical heat field into a practical interface for
multiscale geometric analysis in modern point-cloud learning.
\end{abstract}
\vspace{-.2in}

\section{Introduction}
Modern learning has repeatedly advanced when a branch of mathematics acquires
a practical interface to finite data. Kernels expose functional geometry
through inner products, persistent homology turns topology into stable
summaries, and graph Laplacians make diffusion and spectral structure
computable from samples
~\citep{scholkopf2002learning,edelsbrunner2002topological,coifman2006diffusion}.

For irregular point clouds, however, there is no comparably standard interface
to the rich toolbox of \emph{differential geometry}. Quantities such as local
dimension, anisotropy, curvature, density variation, and geometric transitions
are often precisely what distinguishes classes, yet point-cloud networks
typically learn such structure only implicitly from coordinates and
neighborhoods~\citep{qi2017pointnet,wang2019dynamic,zhao2021point}.

We propose a simple interface: \emph{lift the point cloud to a smooth,
multiscale geometric object}. Formally, \emph{Heat Field Signatures} (HFS)
maps $X=\{x_1,\ldots,x_m\}\subset\mathbb R^n$ to the multiscale ambient heat
field
\[
u_X(y,t)
=
(4\pi t)^{-n/2}
\sum_{i=1}^{m}
\exp\!\left(-\frac{\|y-x_i\|^2}{4t}\right),
\]
whose graph is a smooth hypersurface in $\mathbb R^{n+1}$ for every diffusion
scale $t$. A discrete, irregular sample is thereby replaced by a family of
smooth geometric objects on which derivatives, Hessians, energies, and other
geometric quantities are directly available. The point is not merely to smooth
the data: the heat field provides an \emph{interface through which geometric
analysis can be brought systematically into point-cloud learning}.

This interface is particularly natural for data whose identity is not encoded
by a rigid global silhouette. Neurons, protein structures, spatial processes,
and LiDAR trees are characterized by how geometry changes across space and
scale: filaments meet at junctions, dense regions transition to sparse ones,
and structures of different local dimension coexist. Diffusion time provides
the scale axis needed to resolve this organization, while the smooth heat
landscape makes it analytically accessible. Unlike spectral heat methods that
diffuse \emph{on} a graph or manifold, HFS analyzes the \emph{ambient} field
generated directly by the samples, requiring no mesh, graph Laplacian, normal
estimation, or spectral decomposition.

In this work, we instantiate this broader interface with a small set of
complementary geometric measurements. Global heat signatures capture
concentration, scale response, and roughness. Locally, the \emph{heat
dimension} estimates effective intrinsic dimension, the log-Hessian spectrum
captures tangent--normal structure and anisotropy, and scale-transition
responses identify junctions, boundaries, and mixed-dimensional regions. We
further introduce the \emph{Heat Dimension Spectrum} (HDS), which summarizes
the multiscale distribution of point-, filament-, surface-, and volume-like
regions and their interaction with sampling density. These are not intended
to exhaust the geometry available from the heat field, but to demonstrate the
range of information exposed by the interface.

The resulting representation can be used as a closed-form descriptor, equipped
with a lightweight learned pooling head, or supplied as geometric features to
a point-cloud network; closed-form HFS descriptors are already highly
competitive, while learned pooling and backbone integration provide
additional flexibility. Across synthetic density/topology benchmarks and real
  subcellular, neuronal, tree, and protein data, an HFS variant achieves the
  best performance on all seven benchmarks, while providing substantially
  lower computational cost and exact rotation invariance by construction. On
SCOP protein-fold classification, HFS improves over the strongest deep
point-cloud baseline by nearly $24$ percentage points using coordinates alone.
These results suggest that explicitly exposing geometric structure can be
especially effective when intrinsic multiscale geometry itself carries the
label.

Our contributions are:
\begin{itemize}[nosep,leftmargin=*]
    \item We introduce \textbf{HFS as a geometric interface for point clouds}:
    a discrete cloud is lifted to a multiscale family of smooth heat
    landscapes, making tools from differential geometry directly applicable
    to finite, irregular data.
    \item We instantiate this interface with \textbf{global and local
    heat-geometric signatures} capturing concentration, intrinsic dimension,
    anisotropy, and scale transitions, together with the \textbf{Heat Dimension
    Spectrum (HDS)} for shape-level aggregation.
    \item We develop both \textbf{closed-form and learned HFS
    representations}, retaining geometric invariance while allowing
    task-specific learning and integration with standard point-cloud networks.
    \item Across synthetic and real-world benchmarks, HFS consistently
    outperforms strong point-cloud and multiparameter-persistence baselines
    while providing exact rotation invariance and substantially lower
    computational cost.
\end{itemize}

\section{Related Work}
\label{sec:related}

\noindent \textbf{Point-cloud learning.} \quad
PointNet~\citep{qi2017pointnet} applied shared MLPs and global max-pooling
to individual points; PointNet++~\citep{qi2017pointnet2} added hierarchical
local grouping via farthest-point sampling. DGCNN~\citep{wang2019dynamic}
built dynamic $k$-NN graphs with edge convolutions, and attention-based
architectures such as Point Transformer~\citep{zhao2021point} and
PCT~\citep{guo2021pct} apply self-attention to local neighborhoods; masked
autoencoders~\citep{pang2022masked} further improve representations through
self-supervised pretraining. Density-adaptive variants such as
DRINet~\citep{ye2021drinet} and domain-adaptation methods such as
PointDAN~\citep{qin2019pointdan} address nonuniform sampling by reweighting
neighborhood aggregation, and self-supervised approaches~\citep{huang2021spatio}
learn density invariance through augmentation -- but all of these methods
learn geometric structure implicitly from coordinates and learned features.
In contrast, our approach makes local geometry explicit: rather than
learning density robustness from augmented data, logarithmic scale
derivatives of the heat field are approximately invariant to multiplicative
density variation by construction (Section~\ref{sec:method}), complementing
learned representations with theoretically grounded geometric variables.

\noindent \textbf{Handcrafted descriptors and dimension estimation.} \quad
Before deep learning, point-cloud recognition relied on handcrafted local
descriptors: FPFH~\citep{rusu2009fast} and PFH~\citep{rusu2008aligning}
encode surface geometry through histograms of angular relationships between
point normals, while SHOT~\citep{tombari2010unique} and spin
images~\citep{johnson1999using} combine shape and keypoint-centered
histograms. These operate at a single fixed scale and assume well-oriented
normals, limiting robustness to noisy or heterogeneous sampling.
Separately, intrinsic-dimension estimators -- correlation
dimension~\citep{grassberger1983measuring}, local PCA~\citep{camastra2016intrinsic},
and TWO-NN~\citep{facco2017estimating} -- produce a single scalar dimension
estimate at a fixed scale from pairwise distances. Our heat-geometric
quantities differ from both lines of work in being multiscale by
construction and requiring no normal estimation. Unlike prior dimension
estimators, our heat dimension $d_{\mathrm{heat}}(x,t)$ is a continuous,
scale-dependent profile rather than a single scalar, and is one component
of a joint geometric-type descriptor (with codimension and scale-transition
signals) whose distribution we aggregate into the Heat Dimension Spectrum
-- a shape-level summary of dimensional composition across scale that
existing estimators do not provide.

\noindent \textbf{Heat kernels and spectral methods.} \quad 
The Heat Kernel Signature (HKS)~\citep{sun2009concise} uses the diagonal of
the heat kernel on a mesh as a multiscale point descriptor, exploiting the
connection between heat diffusion and Laplace-Beltrami eigenfunctions; the
Wave Kernel Signature (WKS)~\citep{aubry2011wave} replaces diffusion with
wave propagation for improved frequency selectivity. Both require a mesh
or graph Laplacian on compact surfaces without boundary. Diffusion
maps~\citep{coifman2006diffusion} use the heat kernel for geometry-preserving
embeddings, but target global dimensionality reduction rather than local
geometric type. We differ in three ways: we work directly on unstructured
point clouds without mesh or graph construction; we use the differential
geometry of the ambient heat field rather than spectral decomposition; and
we extract local quantities -- dimension, codimension, and scale
transitions -- rather than global embeddings or signatures.

\noindent \textbf{Topological and geometric data analysis.} \quad
Persistent homology~\citep{edelsbrunner2002topological} encodes multiscale
topological features of point clouds as persistence diagrams; extensions
such as the multicover bifiltration~\citep{edelsbrunner2021multi} incorporate
density information into topological descriptors. The resulting
density--distance bifiltration has also motivated synthetic benchmarks for
evaluating joint sensitivity to scale and
density~\citep{carriere2020multiparameter,kerber2024graphcode}, which we
adopt as evaluation settings for our method. These methods capture
connectivity and homological structure but do not directly measure local
dimension, codimension, or anisotropy. Our approach is complementary:
topological methods capture global connectivity across scale, while
heat-geometric methods capture local geometric type at each point; the two
frameworks could in principle be combined, which we leave for future work.

\section{Heat Field Signatures}
\label{sec:method}

\subsection{Overview}

Given a point cloud
$X=\{x_1,\ldots,x_m\}\subset\mathbb R^n$, we seek a representation that
captures not only \emph{where} the points lie, but how their geometric
organization changes across scale. HFS first converts the discrete cloud
into a smooth multiscale heat field and then reads this field at three
complementary levels:
\[
    X
    \longrightarrow
    u_X(y,t)
    \longrightarrow
    \left[
        \Phi_{\mathrm{global}}(X),\;
        \Phi_{\mathrm{HDS}}(X),\;
        z_X
    \right]
    \longrightarrow
    \text{classifier}.
\]
Global signatures describe how the entire heat landscape concentrates and
dissipates; local signatures identify the geometric type surrounding each
point; and the Heat Dimension Spectrum (HDS) summarizes how these local
types are distributed across the shape. Together they capture both dominant
morphology and localized structures such as thin branches, junctions, and
mixed-dimensional interfaces. Figure~\ref{fig:hfs_pipeline} summarizes the
pipeline; derivations and geometric interpretation are given in
Appendix~\ref{app:method_details}. For general background on geometric heat-kernel
methods, see~\citet{grigoryan2009heat}.

\begin{figure}[t]
\centering
\includegraphics[width=\linewidth]{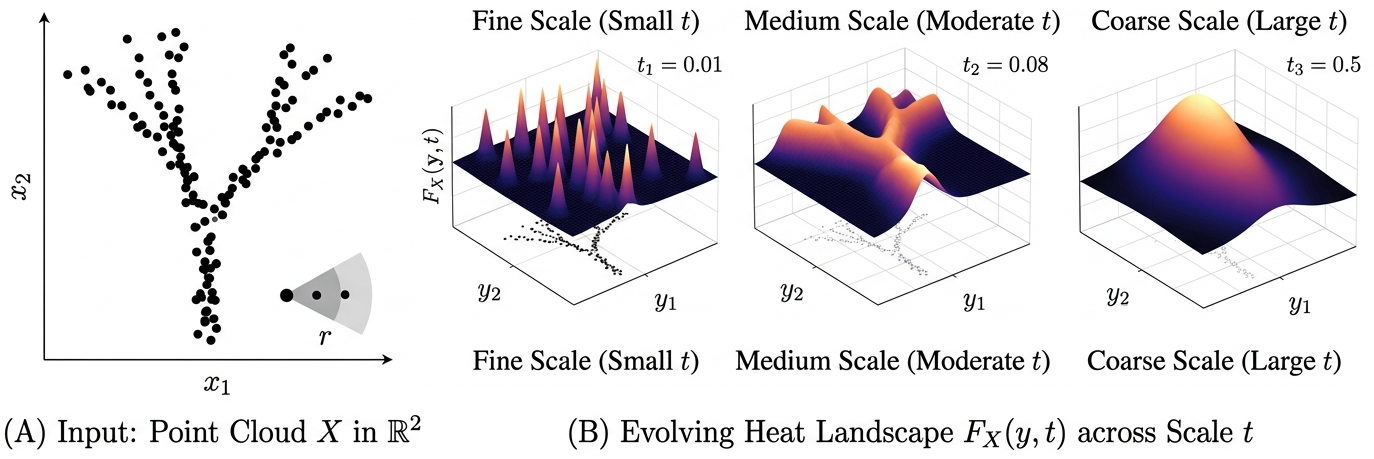}
\caption{\footnotesize
\textbf{Multiscale heat landscape.}
For visualization, we show the unnormalized Gaussian field
$F_X(y,t)=\sum_i e^{-\|y-x_i\|^2/4t}$. Fine scales resolve local
neighborhoods, intermediate scales reveal branches and junctions, and coarse
scales retain global structure. HFS uses the normalized field $u_X$, which
has the same spatial Gaussian structure at each scale.}
\label{fig:heat_landscape}
\vspace{-.2in}
\end{figure}

\subsection{Heat Field and Global Signatures}
\label{sec:global_indicators}

We associate to $X$ the ambient heat field
\quad $    u_X(y,t)
    =
    (4\pi t)^{-n/2}
    \sum_{j=1}^{m}
    \exp\!\left(
        -\frac{\|y-x_j\|^2}{4t}
    \right),
    \quad t>0,$
where $\sqrt t$ acts as a geometric length scale. Small $t$ resolves fine
local structure, while increasing $t$ progressively exposes coarser
organization. Define \quad $    D_{ij}=\|x_i-x_j\|,$ \quad $
    w_{ij}(t)
    =
    \exp\!\left(-\frac{D_{ij}^2}{4t}\right).$ %
Because $u_X$ is a Gaussian mixture, useful global functionals are available
directly from pairwise distances:
{\small 
\[
E_2(t)=(8\pi t)^{-n/2}\!\sum_{i,j}e^{-D_{ij}^2/8t},
\quad
C_2(t)=-t\,\partial_t\log E_2(t),
\quad
\mathcal E(t)=-\tfrac12\partial_tE_2(t),
\quad
\widetilde{\mathcal E}(t)=\frac{\mathcal E(t)}{E_2(t)+\varepsilon}.
\]
}

These measure, respectively, heat concentration, scale response,
spatial roughness, and normalized roughness. Concatenating them across
diffusion scales gives the global descriptor
$\Phi_{\mathrm{global}}(X)$. Their closed forms and geometric
interpretation are detailed in Appendix~\ref{app:method_details}.

\subsection{Local Heat-Geometric Signatures}
\label{sec:local_heat_signatures}

Global quantities describe the heat landscape as a whole. To expose
localized morphology, we instead examine logarithmic derivatives of
$u_X$ at each input point.

\noindent \textbf{Heat dimension.} \quad
We define
\quad $    d_{\mathrm{heat}}(x_i,t)
    =
    n+2t\,\partial_t\log u_X(x_i,t)
    =
    \dfrac{
        \sum_jD_{ij}^2w_{ij}(t)
    }{
        2t\sum_jw_{ij}(t)
    }.
    \label{eq:dheat_closed}
$

If the neighborhood of $x$ behaves locally like a $d$-dimensional
structure, then
\[
    u_X(x,t)
    \approx
    C(x)t^{-(n-d)/2}
    \qquad\Longrightarrow\qquad
    d_{\mathrm{heat}}(x,t)\approx d.
\]
Thus heat dimension provides a smooth, scale-dependent notion of local
intrinsic dimension: in $\mathbb R^3$, values near $1$, $2$, and $3$
correspond ideally to filament-, surface-, and volume-like regions.
\textbf{Theorem}~\ref{thm:heat-dim} in Appendix~\ref{app:heat_dim_intuition}
gives a finite-sample approximation guarantee and identifies the
intermediate regime between sampling resolution and geometric variation
in which this interpretation is valid.

\noindent \textbf{Directional geometry and scale transitions.} \quad
Heat dimension counts effective directions but does not distinguish them.
We therefore use the log-Hessian
\quad $    Q_X(x_i,t)
    =
    \nabla_y^2\log u_X(y,t)\big|_{y=x_i},
$\quad
with eigenvalues
$\lambda_1(x_i,t)\ge\cdots\ge\lambda_n(x_i,t)$.
For an ideal locally flat $d$-dimensional structure,
\[
    2t\lambda_k
    \approx
    \begin{cases}
        0,  & \text{tangent directions},\\
        -1, & \text{normal directions},
    \end{cases}
\]
so the normalized spectrum captures tangent--normal structure,
codimension, and anisotropy.

Finally,
\quad $    \tau(x_i,t)
    =
    \partial_{\log t}d_{\mathrm{heat}}(x_i,t)
    =
    t\,\partial_t d_{\mathrm{heat}}(x_i,t)$ \quad 
measures how rapidly the visible local geometry changes with scale.
Large $|\tau|$ highlights boundaries, junctions, protrusions, and
mixed-dimensional transitions.

Together,
\quad $    \psi_X(x_i,t)
    =
    \left[
        d_{\mathrm{heat}}(x_i,t),\;
        \tau(x_i,t),\;
        2t\lambda_1(x_i,t),\ldots,2t\lambda_n(x_i,t)
    \right]$ \quad
provides a compact local geometric type. In $\mathbb R^3$,
$\psi_X$ has five components per scale. Closed-form derivatives and
a geometric dictionary for these signatures are given in
Appendix~\ref{app:local_heat_geometry}.

\begin{figure}[t]
    \centering
    \includegraphics[width=\linewidth]{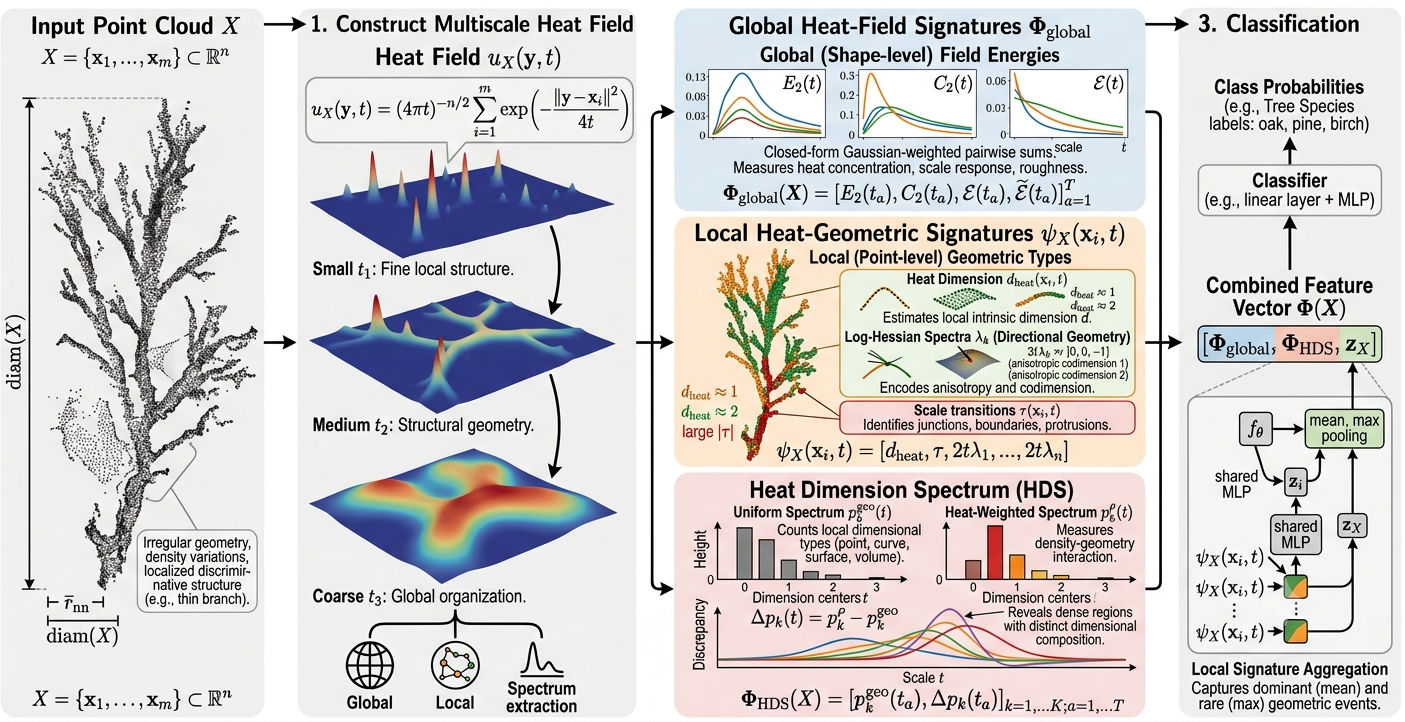}
    \caption{\footnotesize
    \textbf{Heat Field Signatures (HFS).}
    A point cloud is converted into a multiscale ambient heat field.
    HFS extracts complementary global heat responses, local geometric
    signatures, and the Heat Dimension Spectrum (HDS), then aggregates
    them into a fixed-dimensional representation for classification.}
    \label{fig:hfs_pipeline}
\end{figure}

\subsection{Heat Dimension Spectrum}
\label{sec:hds}

A point cloud rarely has a single geometric type: a neuron may contain
filamentary branches and a compact soma, while a spatial process may mix
dense clusters with sparse regions. We therefore summarize the distribution
of local heat dimensions with the \emph{Heat Dimension Spectrum} (HDS).

Let $c_1,\ldots,c_K$ denote dimension centers; in $\mathbb R^3$ we use
$c_k\in\{0,1,2,3\}$. With $\sigma=0.5$, define
\begin{align*}
\eta_k(d)
&=\frac{e^{-(d-c_k)^2/(2\sigma^2)}}
{\sum_{\ell=1}^K e^{-(d-c_\ell)^2/(2\sigma^2)}},
&
p_k^{\mathrm{geo}}(t)
&=\frac1m\sum_i \eta_k(d_{\mathrm{heat}}(x_i,t)),
\\
p_k^\rho(t)
&=\frac{\sum_i u_X(x_i,t)\eta_k(d_{\mathrm{heat}}(x_i,t))}
{\sum_i u_X(x_i,t)},
&
\Delta p_k(t)
&=p_k^\rho(t)-p_k^{\mathrm{geo}}(t).
\end{align*}
Here $p^{\mathrm{geo}}$ records the dimensional composition of the cloud,
while $\Delta p$ measures how that composition changes in heat-concentrated
regions. Concatenating $p^{\mathrm{geo}}$ and $\Delta p$ across scales gives
$\Phi_{\mathrm{HDS}}(X)$. Further motivation and comparison with low-order
summaries are given in Appendix~\ref{app:hds_details}.

\subsection{Vectorization and HFS Variants}
\label{sec:vectorization_classification}

Across $T$ scales, the per-point heat-geometric feature is
\[
\phi(x_i)
=
\left[
u_X(x_i,t_a),\;
d_{\mathrm{heat}}(x_i,t_a),\;
\tau(x_i,t_a),\;
2t_a\lambda_1(x_i,t_a),\ldots,
2t_a\lambda_n(x_i,t_a)
\right]_{a=1}^{T}.
\]
In $\mathbb R^3$, $\phi(x_i)$ has $6T$ components. We evaluate three
increasingly expressive HFS representations:

\begin{itemize}[leftmargin=*]
    \item \textbf{HFS-simple:} a fully closed-form descriptor using the
    global and HDS signatures with fixed statistical pooling of local heat
    features, \textit{excluding} the log-Hessian spectrum.
    \item \textbf{HFS-desc:} the full closed-form descriptor, adding pooled
    log-Hessian features to HFS-simple.
    \item \textbf{HFS-full:} uses the same per-point heat features as
    HFS-desc, but \textit{replaces fixed pooling with a shared MLP} followed by mean
    and max aggregation to obtain $z_X$.
\end{itemize}

Thus the general HFS representation has the form
\[
\boldsymbol{\Phi}(X)
=
\bigl[
\Phi_{\mathrm{global}}(X),\;
\Phi_{\mathrm{HDS}}(X),\;
z_X
\bigr],
\]
where $z_X$ is obtained by fixed statistical pooling for HFS-simple/HFS-desc
and learned pooling for HFS-full. Exact pooling statistics and architectures
are given in Appendix~\ref{app:implementation}.

Unless otherwise stated, we use $T$ logarithmically spaced diffusion times
between
\[
t_{\min}=c_{\min}\bar r_{\mathrm{nn}}^2,
\qquad
t_{\max}=c_{\max}\operatorname{diam}(X)^2.
\]
The endpoints track sampling resolution and coarse shape scale, respectively
(Appendix~\ref{app:scale_choices}). Dense HFS computation costs
$O(m^2T)$ for fixed ambient dimension; scalable approximations are discussed
in Appendix~\ref{app:large_scale_computation}.

\section{Experiments}
\label{sec:experiments}

We evaluate HFS on tasks where class identity depends on multiscale geometry,
density, branching, local dimension, or topological structure. We ask three
primary questions: \emph{(i)} does explicitly encoding this geometry improve
classification, \emph{(ii)} what is the end-to-end computational cost, and
\emph{(iii)} how does the representation behave under arbitrary rotations?
We additionally compare with multiparameter-persistence methods and ablate
the individual HFS components.

\subsection{Experimental Setup}
\label{sec:protocol}

\noindent \textbf{Datasets.} \quad
We use three synthetic density/topology benchmarks,
\textbf{Processes}, \textbf{Orbit5k}, and \textbf{DisksAnnuli}
~\citep{carriere2020multiparameter,kerber2024graphcode}, and four
real-world 3D morphology benchmarks: subcellular structures
(\textbf{Allen}~\citep{vasan2025interpretable,viana2023integrated}),
neurons (\textbf{NeuroMorpho}~\citep{ascoli2007neuromorpho,ascoli2018open}),
LiDAR trees (\textbf{FOR-species20K}~\citep{puliti2025forspecies20k}),
and protein folds (\textbf{SCOP}~\citep{proteinshake}); see
Table~\ref{tab:datasets}. These datasets span density variation, branching
and filamentary structure, multiscale organization, and local geometric type;
SCOP additionally has no canonical orientation and provides a natural
rotation stress test. Full construction, preprocessing, point budgets, and
evaluation protocols are given in Appendix~\ref{app:datasets}. On two
boundary tasks dominated by localized or rigid shape, IntrA and ATLAS-1,
HFS remains competitive but does not outperform the strongest
coordinate-based networks (Appendix~\ref{app:boundary}).

\noindent \textbf{HFS variants.} \quad
\textbf{HFS-simple} and \textbf{HFS-desc} use closed-form, training-free
feature extractors (\Cref{sec:method}). \textbf{HFS-full} adds a learned pooling head over the
per-point HFS features ($\sim\!0.2$M parameters), while
\textbf{HFS+DGCNN} augments a DGCNN backbone with HFS features concatenated
with coordinates. These variants separate the contributions of closed-form
signatures, learned pooling, and backbone integration.

\noindent \textbf{Baselines and evaluation.} \quad
We compare with FPFH~\citep{rusu2009fast}, HKS~\citep{sun2009concise},
PointNet~\citep{qi2017pointnet}, PointNet++~\citep{qi2017pointnet2},
DGCNN~\citep{wang2019dynamic}, and Point Transformer~\citep{zhao2021point}.
All descriptor methods use the same two-layer MLP classifier, while deep
baselines follow their published training recipes. Synthetic benchmarks
report OA over 20 splits (80/20; Orbit5k 70/30), and real benchmarks report
mACC over 5 folds. Rotation-specific baselines are introduced
in Section~\ref{sec:results}.

\noindent \textbf{Implementation.} \quad
All HFS variants use $T=8$ logarithmically spaced diffusion scales with
$c_{\min}=0.05$ and $c_{\max}=0.25$, giving 216 dimensions for HFS-simple,
336 for HFS-desc, and $6T=48$ per-point features in $\mathbb R^3$.
Full hyperparameters, point budgets, augmentation details, and implementation
settings are given in Appendix~\ref{app:implementation}; scale choices are
analyzed in Appendix~\ref{app:scale_choices}. Code and configurations are
available at\\ \url{https://github.com/Wang-Yuanqing/Heat-Field-Signatures}.


\subsection{Results}
\label{sec:results}

\noindent \textbf{Point Cloud Classification.} \quad
Table~\ref{tab:main_results} shows a consistent pattern: an HFS variant
achieves the best performance on all seven benchmarks. On the synthetic
tasks, the closed-form \textbf{HFS-desc} extractor already surpasses all
deep baselines on Processes and Orbit5k, while learned pooling raises HFS-full
to $99.8\%$ on DisksAnnuli. The gains persist on real morphology: HFS-full
or HFS+DGCNN leads on subcellular structures, neurons, and LiDAR trees.
The largest margin occurs on SCOP, where HFS-full reaches $59.6\%$ versus
$35.7\%$ for the strongest deep baseline, a $23.9$-point gain using
$C_\alpha$ coordinates alone.

\begin{table}[t]
\centering
\caption{\footnotesize
\textbf{Point-cloud classification.}  Classical and deep baselines compared with HFS variants.
\textbf{Bold} = best, \underline{underline} = second; HFS rows shaded.}
\label{tab:main_results}
\resizebox{\textwidth}{!}{%
\begin{tabular}{l ccc | cccc}
\toprule
& \multicolumn{3}{c}{\textit{Synthetic (OA over 20 splits)}}
& \multicolumn{4}{c}{\textit{Real (mACC over 5 folds)}} \\
\cmidrule(lr){2-4}\cmidrule(lr){5-8}
\textbf{Method}
    & Processes & Orbit5k & DisksAnnuli
    & Allen & NeuroMorpho & FOR-species & SCOP \\
\midrule
FPFH~\citep{rusu2009fast}          & 84.8\std{1.1} & 30.6\std{1.1} & 36.9\std{2.0} & 90.4\std{0.9} & 58.5\std{1.4} & 46.7\std{2.2} & 24.1\std{1.8} \\
HKS~\citep{sun2009concise}         & 85.2\std{1.2} & 32.8\std{1.1} & 50.4\std{1.8} & 79.5\std{0.8} & 30.7\std{1.3} & 31.2\std{0.8} & 20.0\std{1.1} \\
\midrule
PointNet~\citep{qi2017pointnet}    & 52.2\std{4.1} & 71.6\std{7.3} & 20.0\std{0.8} & 90.1\std{1.3} & 58.9\std{1.4} & 51.6\std{3.1} & 33.9\std{3.9} \\
PointNet++~\citep{qi2017pointnet2} & 59.0\std{3.0} & 89.1\std{1.2} & 91.2\std{2.4} & 92.8\std{0.9} & 64.5\std{2.1} & 62.8\std{1.5} & 19.4\std{1.6} \\
DGCNN~\citep{wang2019dynamic}      & 57.5\std{1.7} & 90.7\std{0.9} & 96.9\std{0.7} & 95.8\std{0.3} & 59.0\std{1.3} & \underline{66.3\std{0.9}} & 35.1\std{2.7} \\
Point Tr~\citep{zhao2021point}     & 60.3\std{1.8} & 88.5\std{1.5} & 82.0\std{31.3} & 95.3\std{0.2} & 63.7\std{1.2} & 57.3\std{1.3} & 35.7\std{2.5} \\
\midrule
\rowcolor{green!8}
HFS-simple                         & 97.3\std{0.7} & 89.3\std{0.7} & 91.4\std{0.9} & 97.0\std{0.4} & 65.4\std{1.5} & 63.5\std{1.5} & 49.3\std{2.8} \\
\rowcolor{green!8}
HFS-desc                           & \textbf{98.3\std{0.4}} & 93.2\std{0.5} & 92.1\std{0.8} & 97.8\std{0.3} & 72.8\std{1.4} & 65.9\std{1.5} & \underline{58.0\std{2.9}} \\
\rowcolor{green!8}
HFS-full                           & \underline{97.5\std{0.4}} & \textbf{97.5\std{0.3}} & \textbf{99.8\std{0.1}} & \underline{97.9\std{0.2}} & \underline{74.1\std{3.3}} & \textbf{67.6\std{1.2}} & \textbf{59.6\std{2.2}} \\
\rowcolor{green!8}
HFS+DGCNN                          & 81.4\std{1.3} & \underline{93.4\std{0.3}} & \underline{98.8\std{0.3}} & \textbf{98.2\std{0.2}} & \textbf{77.6\std{1.1}} & 65.1\std{1.1} & 54.7\std{1.3} \\
\bottomrule
\end{tabular}}
\vspace{-.2in}
\end{table}

\begin{wraptable}{r}{3in}
\centering
\caption{\footnotesize
\textbf{End-to-end runtime (seconds).}
Feature extraction, training, and inference are timed identically on one node
(1 H200 + 16 CPU workers). $\mathcal{S}_{\mathrm{DGCNN}}$ denotes the median
speed-up relative to DGCNN across datasets ($>1$ = faster). Lower is better
for runtime; \textbf{bold} = fastest, \underline{underline} = second; HFS rows shaded.}
\label{tab:time}
\resizebox{\linewidth}{!}{%
\begin{tabular}{l rrrr r}
\toprule
\textbf{Method} & Allen & NMorpho & FOR-spec & SCOP
& $\mathcal{S}_{\mathrm{DGCNN}}$ \\
\midrule
FPFH         & \textbf{33.6} & \textbf{29.6} & \textbf{19.1} & 19.9 & \textbf{43.3} \\
  HKS          & 126.8 & \underline{39.9} & \underline{24.7} & 4.1 & 24.5 \\
  \midrule
  PointNet     & 370.0 & 279.6 & 178.8 & 62.5 & 4.6 \\
  PointNet++   & 1949.6 & 1420.8 & 883.5 & 86.4 & 0.9 \\
  DGCNN        & 1858.8 & 1366.0 & 774.0 & 72.7 & 1.0 \\
  Point Tr     & 1058.6 & 773.2 & 491.7 & 98.5 & 1.7 \\
  \midrule
  \rowcolor{green!8}
  HFS-simple   & \underline{55.0} & 54.2 & 56.0 & \textbf{3.0} & \underline{24.7} \\
  \rowcolor{green!8}
  HFS-desc     & 353.4 & 313.9 & 206.0 & \underline{3.3} & 4.8 \\
  \rowcolor{green!8}
  HFS-full     & 206.9 & 147.7 & 96.1 & 29.2 & 8.5 \\
  \rowcolor{green!8}
  HFS+DGCNN    & 2339.0 & 1804.7 & 984.8 & 74.6 & 0.8 \\
\bottomrule
\end{tabular}}
\vspace{-.2in}
\end{wraptable}

\noindent \textbf{Runtimes.} \quad
Table~\ref{tab:time} reports end-to-end runtime, including feature extraction,
training, and inference. Among the accuracy-competitive methods,
\textbf{HFS-full} is about $8.5\times$ faster than DGCNN in median runtime,

while \textbf{HFS-simple} is about $25\times$ faster and requires no learned
feature extractor. \textbf{HFS-desc} remains $2.4$--$5.5\times$ cheaper than PointNet++, DGCNN and Point Transformer, whereas \textbf{HFS+DGCNN} trades efficiency for its stronger
accuracy on several datasets. Thus the main HFS variants achieve the accuracy
gains of Table~\ref{tab:main_results} at substantially lower computational cost.

\noindent \textbf{Rotation invariance.} \quad
Table~\ref{tab:rotation} evaluates each trained model on the same test clouds
before and after an independent random $SO(3)$ rotation
(Appendix~\ref{app:rotation}). The contrast is striking. Because HFS depends
only on pairwise distances and eigenvalues of the heat-field Hessian, its
underlying representation is invariant by construction. 

\begin{wraptable}{r}{3in}
\vspace{-.2in}
\centering
\caption{\footnotesize
\textbf{Rotation invariance.} Accuracy change ($\Delta$ mACC) under random
$SO(3)$ rotations. HFS is invariant by construction; the small HFS-desc
deviations arise from numerical amplification of nearly constant Hessian
features during standardization. DGCNN-SO(3) and VN-DGCNN obtain robustness
through augmentation and equivariance, respectively.}
\label{tab:rotation}
\resizebox{\linewidth}{!}{%
\begin{tabular}{l cccc c}
\toprule
\textbf{Method} & Allen & NMorpho & FOR-spec & SCOP & Invariance \\
\midrule
FPFH         & 6.4  & 8.0  & 2.3  & $-1.0$ & approximate \\
HKS     & 28.0 & 7.1  & 12.5 & 2.3    & approximate \\
\midrule
PointNet   & 35.8 & 38.0 & 39.7 & $-2.0$ & none \\
PointNet++ & 11.1 & 30.9 & 45.1 & $-1.0$ & none \\
DGCNN      & 19.0 & 31.5 & 47.8 & \textbf{0.0} & none \\
Point Tr     & 21.0 & 34.7 & 40.7 & \textbf{0.0} & none \\
\midrule
DGCNN-SO(3)                    & $-0.2$ & 0.6 & 0.6 & \textbf{0.0} & augmented \\
VN-DGCNN    & 0.1 & \textbf{0.0} & \textbf{0.0} & \textbf{0.0} & architectural \\
\midrule
\rowcolor{green!8}
HFS-simple                         & \textbf{0.0} & \textbf{0.0} & \textbf{0.0} & \textbf{0.0} & exact  \\
\rowcolor{green!8}
HFS-desc                           & \textbf{0.0} & $-0.8$ & $-0.1$ & \textbf{0.0} & $\sim$exact  \\
\rowcolor{green!8}
HFS-full                           & \textbf{0.0} & \textbf{0.0} & \textbf{0.0} & \textbf{0.0} & exact  \\
\rowcolor{green!8}
HFS+DGCNN                          & $-0.1$ & 1.1 & $-0.7$ & 0.7 & inherited  \\
\bottomrule
\end{tabular}}
\vspace{-.1in}
\end{wraptable}
HFS-simple and
HFS-full therefore show no accuracy change under rotation, while HFS-desc
exhibits only sub-point numerical fluctuations. We traced these small
deviations to Hessian-eigenvalue percentile features that become nearly
constant across clouds at the finest scales: standardization can amplify
floating-point perturbations from eigendecomposition enough to change a few
classifier predictions. Repeated runs produced fluctuations of comparable
magnitude, confirming that this is a numerical conditioning effect rather
than a failure of the geometric invariance. In contrast, standard
coordinate-based networks lose $11$--$48$ mACC points on Allen,
NeuroMorpho, and FOR-species, often erasing much of their original
performance. Moreover, the HFS guarantee is not limited to $SO(3)$: it
extends to the full Euclidean group $O(n)\ltimes\mathbb R^n$, including
rotations, reflections, and translations, in any ambient dimension.

We also compare with two approaches designed specifically for rotation
robustness. \textbf{DGCNN-SO(3)}~\citep{esteves2018learning} retrains DGCNN
with Haar $SO(3)$ augmentation, while
\textbf{VN-DGCNN}~\citep{deng2021vector} uses the rotation-equivariant Vector
Neuron architecture. Both largely eliminate the rotation drop, but through
augmentation or architectural constraints rather than an invariant input
representation. Interestingly, adding HFS features to the otherwise
rotation-sensitive DGCNN almost completely removes its degradation:
HFS+DGCNN changes by at most $1.1$ points across the four datasets, compared
with losses of up to $47.8$ points for DGCNN alone. This suggests that the
invariant HFS channel can strongly stabilize a coordinate-based backbone,
although this is an empirical effect rather than a formal guarantee---the HFS
features may simply dominate the hybrid prediction. VN-DGCNN also remains
  well below the best HFS variant on the same split: clean mACC
  $94.1/51.7/55.4/55.3$ versus $98.2/77.6/68.1/63.0$ on
  Allen/NeuroMorpho/FOR-species/SCOP. SCOP provides an instructive
contrast: conventional networks are already relatively insensitive to rotation
because protein structures have no canonical reference frame. Thus robustness
on SCOP largely reflects the data distribution, whereas HFS invariance is a
representation-level property that holds independently of the dataset.

\noindent \textbf{Comparison with multiparameter persistence.} \quad
Processes and DisksAnnuli originate in the multiparameter-persistence
literature, where density and geometric scale are modeled jointly, making
them particularly relevant benchmarks for HFS. We compare against MP-Images, MP-Landscapes,
GRIL, MP-HSM-C (Signed Barcodes), and
GraphCode variants (GC, GC-NE), using the
published results of~\citet{kerber2024graphcode} under the same 20-split
evaluation protocol. Further background on the density--distance
bifiltration is given in Appendix~\ref{app:mp}.

\begin{wraptable}{r}{3in}
\vspace{-.2in}
\centering
\caption{\footnotesize \textbf{Multiparameter-persistence baselines.}
Accuracy (\%) on the synthetic benchmarks; MP baseline numbers are
  from~\citet{kerber2024graphcode}.}
\label{tab:mp_comparison}
\setlength{\tabcolsep}{4pt}
\resizebox{\linewidth}{!}{%
\begin{tabular}{l cc}
\toprule
\textbf{Method} & DisksAnnuli & Processes \\
\midrule
MP-I~\citep{carriere2020multiparameter}   & 64.1 & 66.0 \\
MP-L~\citep{vipond2020multiparameter}     & 37.2 & 50.2 \\
GRIL~\citep{xin2023gril}                  & 74.9 & 61.1 \\
MP-HSM-C~\citep{loiseaux2023stable}       & 57.0 & 70.7 \\
GC~\citep{kerber2024graphcode}            & 86.9 & 83.4 \\
  GC-NE~\citep{kerber2024graphcode}         & 82.8 & 83.1 \\
  \midrule
  \rowcolor{green!8}
  HFS-desc & \underline{92.1} & \textbf{98.3} \\
  \rowcolor{green!8}
  HFS-full & \textbf{99.8} & \underline{97.5} \\
\bottomrule
\end{tabular}}
\vspace{-.1in}
\end{wraptable}
Table~\ref{tab:mp_comparison} shows a large gap in favor of HFS:
HFS-full reaches $99.8\%$ on DisksAnnuli versus $86.9\%$ for GraphCode,
while the closed-form HFS-desc extractor reaches $98.3\%$ on Processes
versus $83.4\%$. The striking point is that HFS captures the same
density--geometry interaction without constructing or vectorizing a
persistence module. HFS and persistence remain conceptually complementary,
but on these density-aware benchmarks the heat-field representation is both
simpler and substantially more accurate.

\noindent \textbf{Ablation Studies.} \quad
Table~\ref{tab:ablation} isolates which parts of the representation drive
the gains in Table~\ref{tab:main_results}. All ablations use the
\emph{closed-form} HFS-desc descriptor rather than HFS-full: the latter adds
a learned pooling head over the same six blocks, so ablating each block would
require retraining the head for all thirteen configurations, confounding
representation content with training noise. The closed-form descriptor isolates
the former cleanly. Here $\diamond$ denotes the complete $336$-dimensional
descriptor---\emph{not} HFS-full---formed by six blocks: global heat-field
energies, the heat dimension spectrum (HDS), and pooled statistics of four
local channels (density $u_X$, heat dimension $d_{\mathrm{heat}}$, transition
rate $\tau$, and the log-Hessian spectrum). We ablate each block in both
directions: keeping only that block (\emph{X only}) or removing it
($\diamond,-,$X). Since these six blocks form exactly the descriptor in
Table~\ref{tab:main_results}, $\diamond$ reproduces HFS-desc and
$\diamond,-,$log-Hessian reproduces HFS-simple; the two tables agree within
$0.7$ points throughout, providing an internal consistency check.
Three findings stand out.

First, \textbf{no single block suffices}: on every
dataset the best individual block trails the full descriptor, by as much as
$4.8$ points (FOR-species), so the gains come from combining channels rather
than from one dominant feature. Second, \textbf{which block matters is
dataset-dependent, and in the direction the method predicts}: the density
channel alone nearly closes the gap on DisksAnnuli ($89.0$ of $92.2$), whose
classes are defined by sampling density, yet collapses on the subcellular
Allen data ($27.5$), where density carries little class information;
conversely, the log-Hessian block, which encodes local anisotropy, is the
strongest single block on every real dataset. A descriptor tuned to one
geometric cue could not reproduce this pattern. Third, \textbf{the
leave-one-out direction shows redundancy}: removing any single block costs
at most $1.7$ points, because each channel is a different derivative of the
same heat field and the others partially recover the lost signal. The
exceptions are informative rather than incidental: removing the log-Hessian
block costs $8.4$ points on SCOP and $6.2$ on NeuroMorpho -- the two
datasets whose labels are carried by filamentary structure -- and removing
all four local blocks together costs up to $10.7$ points. The
scale-schedule ablation ($T$, $c_{\min}$, $c_{\max}$) is reported in
Appendix~\ref{app:ablations}.

\begin{table}[t]
\centering
\caption{\footnotesize \textbf{Component ablation.} We report two types of
ablations: the performance of each part alone, and of all parts but one.
$\diamond=$ HFS-desc and $\diamond\,-\,$log-Hessian $=$
HFS-simple of Table~\ref{tab:main_results}. \textbf{Bold} = full
descriptor; \underline{underline} = best single block per column. Avg.\ is
the unweighted mean across all seven benchmarks and is not part of the
bold/underline scheme. Same classifier, splits, and folds as main table.}
\label{tab:ablation}
\resizebox{\textwidth}{!}{%
\begin{tabular}{c | l r ccc | cccc | c}
\toprule
& & & \multicolumn{3}{c}{\textit{Synthetic (OA)}} & \multicolumn{4}{c}{\textit{Real (mAcc)}} & \\
\cmidrule(lr){4-6}\cmidrule(lr){7-10}
\textbf{Ablation type} & \textbf{Configuration} & dim & Proc. & Orbit5k & DisksAnn.
    & Allen & NMorpho & FOR-s.. & SCOP & \textbf{Avg.} \\
\midrule
\rowcolor{green!8}
\textit{Full descriptor} & HFS-desc ($\diamond$)      & 336 & \textbf{98.3} & \textbf{93.3} & \textbf{92.2} & \textbf{97.8} & \textbf{72.3} & \textbf{65.8} & \textbf{58.0} & 82.5 \\
\midrule
 & global only              &  32 & 97.5 & 65.5 & 54.2 & 90.5 & 64.3 & 58.4 & 41.3 & 67.4 \\
 & HDS only                 &  64 & 96.4 & 87.1 & 85.1 & 96.4 & 59.7 & 58.8 & 42.7 & 75.2 \\
\textbf{Single block} & density only             &  40 & 97.0 & 21.8 & \underline{89.0} & 27.5 & 35.7 & 40.8 & 42.1 & 50.6 \\
\textbf{only} & heat-dim only            &  40 & 95.4 & 84.8 & 76.3 & 95.9 & 58.3 & 59.0 & 43.7 & 73.3 \\
 & transition only          &  40 & 94.9 & 85.6 & 74.3 & 96.2 & 59.2 & 57.7 & 44.3 & 73.2 \\
 & log-Hessian only         & 120 & \underline{97.8} & \underline{90.2} & 83.2 & \underline{97.2} & \underline{68.1} & \underline{61.0} & \underline{54.6} & 78.9 \\
\midrule
 & $\diamond$ $-$ global          & 304 & 98.3 & 93.4 & 91.8 & 97.4 & 71.0 & 65.4 & 58.0 & 82.2 \\
 & $\diamond$ $-$ HDS             & 272 & 98.3 & 92.1 & 91.2 & 97.8 & 72.1 & 65.4 & 56.8 & 82.0 \\
\textbf{Full minus}  & $\diamond$ $-$ density         & 296 & 98.2 & 93.3 & 90.5 & 97.8 & 72.0 & 66.1 & 58.2 & 82.3 \\
\textbf{one block} & $\diamond$ $-$ heat-dim        & 296 & 98.3 & 93.0 & 92.1 & 97.9 & 73.1 & 66.8 & 57.6 & 82.7 \\
 & $\diamond$ $-$ transition      & 296 & 98.3 & 93.0 & 92.1 & 97.8 & 72.8 & 65.8 & 58.3 & 82.6 \\
 & $\diamond$ $-$ log-Hessian & 216 & 97.3 & 89.4 & 91.4 & 97.3 & 66.1 & 63.5 & 49.6 & 79.2 \\
 & $\diamond$ $-$ all four local blocks &  96 & 97.6 & 87.8 & 87.0 & 96.7 & 64.4 & 61.9 & 47.3 & 77.5 \\
\bottomrule
\end{tabular}}
\vspace{-.2in}
\end{table}

\subsection{Discussion}
\label{sec:discussion}

Taken together, the experiments support a simple conclusion: when class
identity is carried by intrinsic multiscale geometry, it can be advantageous
to \emph{compute} that geometry rather than ask a network to discover it from
coordinates. Across seven benchmarks, HFS combines strong classification
performance with substantially lower computational cost and exact geometric
invariance. The result is especially striking because much of this advantage
is already present in the closed-form HFS descriptor; learned pooling
  (HFS-full) and backbone integration (HFS+DGCNN) improve it further, but they
  are not what creates the geometric signal.

More broadly, HFS illustrates a different design principle for point-cloud
learning. Instead of treating geometry as an implicit intermediate
representation to be learned, we first lift the discrete cloud to a smooth
multiscale field and then extract quantities with direct geometric meaning:
concentration, intrinsic dimension, anisotropy, and scale transitions.
The ablations show that these cues are complementary rather than redundant,
while the rotation experiments show a second benefit of explicit structure:
properties that require augmentation or specialized architectures in
coordinate networks can arise automatically from the representation itself.
This suggests that classical geometric analysis and modern learning need not
be competing approaches; analytic geometry can provide the representation,
with learning reserved for the task-specific readout.

The claim is deliberately not universal. On tasks dominated by rigid global
shape or localized appearance, coordinate-based networks remain competitive
or stronger (Appendix~\ref{app:boundary}). HFS targets a complementary regime:
irregular data where density, scale, local dimension, and morphology are
entangled. The broader opportunity is therefore not simply a new point-cloud
descriptor, but a family of representations obtained by interrogating
multiscale fields with geometric invariants---of which the signatures studied
here are only a first set.

\vspace{-.1in}

\section{Conclusion}
\label{sec:conclusion}

We introduced \emph{Heat Field Signatures} (HFS), a multiscale interface that
lifts an irregular point cloud to a family of smooth ambient heat fields,
making tools from geometric analysis directly accessible to point-cloud
learning. The specific signatures developed here provide closed-form global and
local geometric descriptors, require no graph, mesh, or persistence
computation, and integrate naturally with lightweight or deep classifiers.
Across diverse synthetic and real-world benchmarks, the results show that this
heat-field interface provides an effective representation for irregular
geometric data.
Several directions remain open. Sparse kernels and landmarks could scale HFS
to very large point clouds, while richer differential and integral invariants
could capture geometry beyond the signatures studied here. HFS could also be
combined with persistent homology to couple local geometric type with global
connectivity. More broadly, HFS suggests a general strategy for representation
learning: lift discrete data to a smooth multiscale geometric object, then apply
the geometric tools best suited to the task.






\subsubsection*{Acknowledgments}
This work was partially supported by National Science Foundation under grants DMS-2220613, and DMS-2229417, and Simons Foundation under grant MPS-SFM-21269.

\bibliography{references}
\bibliographystyle{iclr2027_conference}

\clearpage
\appendix

\section*{Appendix}

\section{Heat Field Signatures: Details and Geometric Interpretation}
\label{app:method_details}

This section collects the derivations and geometric interpretation of the
heat-field quantities used in the main paper. The central idea is simple:
Gaussian smoothing turns a discrete point cloud into a differentiable
multiscale landscape. Global functionals describe the landscape as a whole,
while logarithmic derivatives at the sample points describe local geometric
type. For general background on heat-kernel
methods and geometric analysis, see~\cite{grigoryan2009heat,evans2010pde,docarmo1992riemannian}.

\subsection{From Point Clouds to Heat Fields}
\label{app:ball_count_to_heat}
\label{app:heat_landscape}

A hard probe of local geometry is the ball count
\[
    N_X(y,r)=\#\{x_i\in X:\|y-x_i\|\le r\}.
\]
If the data near $y$ are approximately $d$-dimensional, then
$N_X(y,r)\propto r^d$ over an appropriate scale range. HFS replaces this
discontinuous count by the Gaussian-smoothed field
\[
    u_X(y,t)
    =
    (4\pi t)^{-n/2}
    \sum_{i=1}^{m}
    \exp\!\left(-\frac{\|y-x_i\|^2}{4t}\right),
    \qquad t>0,
\]
where $\sqrt t$ acts as a spatial scale. The normalization makes
$u_X$ satisfy the heat equation
\[
    \partial_tu_X=\Delta_yu_X.
\]

Geometrically, small $t$ resolves individual points and fine local
structure; intermediate scales reveal branches, sheets, clusters, and
junctions; and large $t$ retains only coarse organization. HFS measures
this evolution without explicitly constructing a mesh, graph Laplacian,
or surface.

\subsection{Global Heat-Field Signatures}
\label{app:global_heat_geometry}
\label{app:energy}
\label{app:capacity}
\label{app:dirichlet}
\label{app:summary}

Let $D_{ij}=\|x_i-x_j\|$. Since $u_X$ is a Gaussian mixture, its principal
global functionals have closed forms:
\begin{align}
    E_2(t)
    &=
    \int_{\mathbb R^n}u_X(y,t)^2\,dy
    =
    (8\pi t)^{-n/2}
    \sum_{i,j}
    e^{-D_{ij}^2/8t},
    \label{eq:app_E2}
    \\
    C_2(t)
    &=
    -t\,\partial_t\log E_2(t)
    =
    \frac n2
    -
    \frac{
        \sum_{i,j}\frac{D_{ij}^2}{8t}e^{-D_{ij}^2/8t}
    }{
        \sum_{i,j}e^{-D_{ij}^2/8t}
    },
    \label{eq:app_C2}
    \\
    \mathcal E(t)
    &=
    \int_{\mathbb R^n}\|\nabla u_X(y,t)\|^2\,dy
    =
    -\frac12\partial_tE_2(t),
    \label{eq:app_dirichlet}
    \\
    \widetilde{\mathcal E}(t)
    &=
    \frac{\mathcal E(t)}{E_2(t)+\varepsilon}.
    \label{eq:app_dirichlet_norm}
\end{align}
These quantities answer complementary questions. $E_2$ measures
\emph{concentration}: nearby heat kernels overlap strongly and increase the
energy. $C_2$ measures the rate at which this concentration changes with
scale. In an ideal $d$-dimensional regime,
\[
    E_2(t)\sim t^{-(n-d)/2},
    \qquad
    C_2(t)\approx\frac{n-d}{2},
\]
so $C_2$ is globally sensitive to codimension. Finally,
$\mathcal E$ measures spatial roughness of the heat landscape, while
$\widetilde{\mathcal E}$ removes its overall energy scale.

Thus the global block records not simply the size of the point cloud, but
how concentration and roughness evolve as neighboring structures begin to
interact.

\paragraph{Heat bodies for visualization.}
\label{app:heat_bodies_method}
\label{app:heat_body}
\label{app:volume}
\label{app:heat_body_compactness}
For qualitative visualization only, we also consider superlevel sets
\[
    \Omega_{\alpha,t}(X)
    =
    \{y:u_X(y,t)\ge \alpha M_X(t)\},
    \qquad
    M_X(t)=\max_i u_X(x_i,t).
\]
At suitable scales these behave like smooth neighborhoods of the data:
filaments produce tube-like regions, surfaces produce slab-like regions,
and disconnected structures merge as $t$ increases. Because their volume
and boundary geometry require grid evaluation, heat bodies are not used in
the main closed-form HFS representation.

\subsection{Local Heat Geometry}
\label{app:local_heat_geometry}
\label{app:local_intro}

Global signatures summarize the whole heat landscape, but many
discriminative structures---thin branches, junctions, boundaries, and
mixed-dimensional interfaces---occupy only a small subset of points.
HFS therefore evaluates logarithmic derivatives of the heat field at each
$x_i$.

Let
\[
    w_{ij}(t)
    =
    \exp\!\left(-\frac{D_{ij}^2}{4t}\right),
    \qquad
    W_i(t)=\sum_jw_{ij}(t).
\]

\paragraph{Heat dimension.}
\label{app:heat_dim_intuition}
The heat dimension is
\[
    d_{\mathrm{heat}}(x_i,t)
    =
    n+2t\,\partial_t\log u_X(x_i,t)
    =
    \frac{
        \sum_jD_{ij}^2w_{ij}(t)
    }{
        2t\,W_i(t)
    }.
    \label{eq:app_dheat_closed}
\]
It is a smooth analogue of the local mass-growth exponent. Indeed, if a
neighborhood behaves like a $d$-dimensional set,
\[
    u_X(x,t)\approx C(x)t^{-(n-d)/2},
\]
then
\[
    d_{\mathrm{heat}}(x,t)\approx d.
\]
Thus in $\mathbb R^3$ the values $1,2,3$ correspond ideally to
filament-, surface-, and volume-like neighborhoods. At scales below the
sampling resolution the self-point can dominate and
$d_{\mathrm{heat}}\approx0$, which motivates avoiding excessively small
$t$.

\paragraph{Finite-sample guarantee: informal statement.}
At a smooth point of a $d$-dimensional manifold,
$d_{\mathrm{heat}}(x,t)$ is close to $d$ when the heat scale lies between
the sampling scale and the geometric variation scale. Up to logarithmic
factors, the error has the form
\[
    \underbrace{O(t)}_{\text{geometry/density bias}}
    +
    \underbrace{
    O\!\left(
        \frac{1}{\sqrt{m\,p(x)\,t^{d/2}}}
    \right)}_{\text{finite-sample noise}},
\]
together with a smaller correction caused by the self-point.

\begin{theorem}[Finite-sample approximation of heat dimension]
\label{thm:heat-dim}
Let $M\subset\mathbb R^n$ be a compact $C^3$ embedded submanifold without
boundary, of intrinsic dimension $d<n$, with
$\operatorname{reach}(M)\ge\tau>0$. Let $p$ be a $C^2$ density on $M$
satisfying
\[
    0<p_{\min}\le p\le p_{\max},
    \qquad
    \|\nabla_Mp\|\le Lp_{\max},
    \qquad
    \|\nabla_M^2p\|\le L^2p_{\max}.
\]
Fix $x\in M$, set $x_1=x$, and draw $x_2,\ldots,x_m$ independently from
$p$. Define
\[
    \mu(x,t)
    =
    (m-1)p(x)(4\pi t)^{d/2}.
\]
There exist constants $c,C_0,C_1,C_2,C_3>0$, depending only on
$d,n$, and $p_{\max}/p_{\min}$, such that if
\[
    t(\tau^{-1}+L)^2\le c,
    \qquad
    \mu(x,t)\ge C_0\log\frac{2}{\delta},
\]
then, with probability at least $1-\delta$,
\[
    \bigl|d_{\mathrm{heat}}(x,t)-d\bigr|
    \le
    C_1t(\tau^{-1}+L)^2
    +
    C_2
    \sqrt{\frac{\log(2/\delta)}{\mu(x,t)}}
    +
    \frac{C_3}{\mu(x,t)}.
\]
\end{theorem}

The theorem identifies the scale window
\quad $    \left(
        \dfrac{\log(2/\delta)}{m\,p(x)}
    \right)^{2/d}
    \lesssim
    t
    \ll
    (\tau^{-1}+L)^{-2},$ \quad 
or, more intuitively,
\quad $   \text{sampling scale}
    \;\lesssim\;
    \sqrt t
    \;\ll\;
    \text{geometric variation scale}.$
This also motivates the practical choice
$t_{\min}\propto\bar r_{\mathrm{nn}}^2$ used in HFS.

\begin{proof}
Write
\[
    N_0=\sum_{j=2}^{m}e^{-\|x-x_j\|^2/4t},
    \qquad
    S=\sum_{j=2}^{m}
        \|x-x_j\|^2e^{-\|x-x_j\|^2/4t}.
\]
Since the self-point contributes one to the denominator and zero to the
numerator,
\[
    d_{\mathrm{heat}}(x,t)=\frac{S}{2t(1+N_0)}.
\]

For $\sqrt t$ small relative to the reach and density-variation scale,
standard local-coordinate estimates approximate $M$ by
$T_xM\simeq\mathbb R^d$. Using
\[
    \int_{\mathbb R^d}e^{-\|u\|^2/4t}\,du=(4\pi t)^{d/2},
    \qquad
    \int_{\mathbb R^d}
        \|u\|^2e^{-\|u\|^2/4t}\,du
        =2dt(4\pi t)^{d/2},
\]
gives
\begin{align*}
    \mathbb E[N_0]
    &=
    \mu(x,t)
    \left[
        1+O\!\left(t(\tau^{-1}+L)^2\right)
    \right],
    \\
    \mathbb E[S]
    &=
    2dt\,\mu(x,t)
    \left[
        1+O\!\left(t(\tau^{-1}+L)^2\right)
    \right].
\end{align*}
Hence
\[
    \frac{\mathbb E[S]}
         {2t\,\mathbb E[N_0]}
    =
    d+O\!\left(t(\tau^{-1}+L)^2\right).
\]

The summands of $N_0$ are bounded by $1$, while
$r^2e^{-r^2/4t}\le4t/e$. Bernstein's inequality therefore yields,
simultaneously with probability at least $1-\delta$,
\[
    |N_0-\mathbb E N_0|
    \le
    C\sqrt{\mu(x,t)\log(2/\delta)},
\]
and
\[
    |S-\mathbb E S|
    \le
    Ct\sqrt{\mu(x,t)\log(2/\delta)},
\]
after absorbing lower-order logarithmic terms under
$\mu(x,t)\ge C_0\log(2/\delta)$. Propagating these deviations through
the ratio gives
\[
    \left|
        \frac{S}{2tN_0}
        -
        \frac{\mathbb E S}{2t\,\mathbb E N_0}
    \right|
    \le
    C\sqrt{\frac{\log(2/\delta)}{\mu(x,t)}}.
\]
Finally,
\[
    \left|
        \frac{S}{2t(1+N_0)}
        -
        \frac{S}{2tN_0}
    \right|
    =
    \frac{S}{2tN_0(1+N_0)}
    \le
    \frac{C}{\mu(x,t)},
\]
which gives the stated bound.
\end{proof}

The theorem concerns smooth manifold regions. Junctions, boundaries, and
mixed-dimensional interfaces need not satisfy its assumptions; for HFS,
such departures are useful signals rather than failure cases.

\paragraph{Log-Hessian spectrum.}
\label{app:local_derivatives}
\label{app:hessian_intuition}
Heat dimension counts effective directions but does not describe their
orientation. For this we use
\[
    Q_X(x_i,t)
    =
    \nabla_y^2\log u_X(y,t)\big|_{y=x_i}.
\]
If
\[
    \bar r_i
    =
    \frac{1}{W_i(t)}
    \sum_jw_{ij}(t)(x_i-x_j),
\]
then the closed form is
\[
    Q_X(x_i,t)
    =
    \frac{1}{4t^2}
    \left[
        \frac{1}{W_i(t)}
        \sum_jw_{ij}(t)(x_i-x_j)(x_i-x_j)^\top
        -
        \bar r_i\bar r_i^\top
    \right]
    -
    \frac{I}{2t}.
    \label{eq:app_log_hessian_closed}
\]
Thus the log-Hessian is simply a scale-normalized local weighted covariance
minus the ambient isotropic heat term.

Let
\[
    \lambda_1(x_i,t)\ge\cdots\ge\lambda_n(x_i,t)
\]
be its eigenvalues. For an ideal locally flat $d$-dimensional structure,
\[
    2t\lambda_k
    \approx
    \begin{cases}
        0, & \text{tangent directions},\\
        -1, & \text{normal directions}.
    \end{cases}
\]
In $\mathbb R^3$ this gives
\[
\begin{array}{c|c}
\text{local type}
&
(2t\lambda_1,2t\lambda_2,2t\lambda_3)
\\ \hline
\text{filament}
&
(0,-1,-1)
\\
\text{surface}
&
(0,0,-1)
\\
\text{volume}
&
(0,0,0).
\end{array}
\]
Departures from these ideal patterns encode curvature, anisotropy,
branching, boundaries, and other nonhomogeneous structure.

\paragraph{Scale-transition rate.}
\label{app:tau_intuition}
Finally,
\[
    \tau(x_i,t)
    =
    t\,\partial_t d_{\mathrm{heat}}(x_i,t)
    =
    \partial_{\log t}d_{\mathrm{heat}}(x_i,t)
\]
measures how quickly the visible local dimension changes with scale (computed
  by finite differences in $\log t$ across the $T$ scales).
Homogeneous interior regions have $\tau\approx0$. A point on a thin branch
attached to a surface, for example, can transition from
$d_{\mathrm{heat}}\approx1$ to $d_{\mathrm{heat}}\approx2$ as the heat
radius grows, producing large $|\tau|$. Junctions and boundaries similarly
generate transient responses.

The three local quantities therefore have distinct roles:
\[
\boxed{
\begin{aligned}
d_{\mathrm{heat}}
    &:\ \text{how many geometric directions?}\\
2t\lambda_k
    &:\ \text{what tangent--normal structure?}\\
\tau
    &:\ \text{how does the geometry change with scale?}
\end{aligned}}
\]
\label{app:local_work_together}
\label{app:local_summary}

\paragraph{Density-sensitive versus geometry-sensitive channels.}
\label{app:density_geometry_channels}
Raw heat values depend directly on local sampling density. In contrast,
if locally
\[
    u_X(x,t)
    \approx
    \rho(x)t^{-(n-d)/2}\times\mathrm{const},
\]
then
\[
    \log u_X(x,t)
    \approx
    \log\rho(x)
    -
    \frac{n-d}{2}\log t
    +
    \mathrm{const}.
\]
A slowly varying multiplicative density factor therefore largely cancels
from derivatives of $\log u_X$. This motivates separating
\[
    \phi^\rho(x_i)
    =
    [u_X(x_i,t_a)]_{a=1}^T
\]
from
\[
    \phi^g(x_i)
    =
    [
        d_{\mathrm{heat}}(x_i,t_a),
        \tau(x_i,t_a),
        2t_a\lambda_1(x_i,t_a),
        \ldots,
        2t_a\lambda_n(x_i,t_a)
    ]_{a=1}^T.
\]
The former deliberately retains density information; the latter emphasizes
support geometry, although strong density gradients at the same scale as
the kernel can still affect the geometric channels.

\subsection{Heat Dimension Spectrum}
\label{app:hds_details}
\label{app:dimension_stats}

The local descriptor is pointwise, whereas classification requires a
shape-level summary. HDS records the distribution of local heat dimensions
across the cloud and across scale.

Let $c_1,\ldots,c_K$ be dimension centers; in $\mathbb R^3$ we use
$c_k\in\{0,1,2,3\}$. Define
\[
    \eta_k(d)
    =
    \frac{
        \exp\!\left(-(d-c_k)^2/(2\sigma^2)\right)
    }{
        \sum_{\ell=1}^K
        \exp\!\left(-(d-c_\ell)^2/(2\sigma^2)\right)
    },
    \qquad \sigma=0.5.
\]
The uniform spectrum and heat-weighted spectrum are
\[
    p_k^{\mathrm{geo}}(t)
    =
    \frac1m
    \sum_i
    \eta_k(d_{\mathrm{heat}}(x_i,t)),
\]
and
\[
    p_k^\rho(t)
    =
    \frac{
        \sum_i
        u_X(x_i,t)
        \eta_k(d_{\mathrm{heat}}(x_i,t))
    }{
        \sum_i u_X(x_i,t)
    }.
\]
Their difference
\[
    \Delta p_k(t)
    =
    p_k^\rho(t)-p_k^{\mathrm{geo}}(t)
\]
measures whether heat-concentrated regions have a different dimensional
composition from the cloud as a whole. The final descriptor is
\[
    \Phi_{\mathrm{HDS}}(X)
    =
    \left[
        p_k^{\mathrm{geo}}(t_a),
        \Delta p_k(t_a)
    \right]_
    {k=1,\ldots,K;\;a=1,\ldots,T}.
\]

HDS retains information that mean dimension alone loses: for example,
a half-filament/half-volume cloud and a purely surface-like cloud can both
have mean dimension $2$, while their spectra are clearly different.
For $K=4$ and $T=8$, HDS has only $64$ entries and costs
$O(mKT)$ once the heat quantities are available.

\subsection{Scale Choice and Computation}
\label{app:scale_choices}
\label{app:large_scale_computation}

We use $T$ logarithmically spaced scales between
\[
    t_{\min}
    =
    c_{\min}\bar r_{\mathrm{nn}}^2,
    \qquad
    t_{\max}
    =
    c_{\max}\operatorname{diam}(X)^2.
\]
The lower endpoint tracks sampling resolution, consistent with
Theorem~\ref{thm:heat-dim}, while the upper endpoint captures coarse
shape organization. For strongly nonuniform clouds we additionally consider
the adaptive alternative
\[
    t_{i,a}=c_a r_i(k)^2,
\]
where $r_i(k)$ is the $k$-th-neighbor distance.

Dense pairwise computation costs $O(m^2T)$ for fixed ambient dimension.
Because Gaussian contributions beyond radius $C\sqrt t$ decay
exponentially, a truncated implementation reduces this to $O(mkT)$ when
only $k$ local neighbors contribute appreciably.

\begin{table}[t]
\centering
\caption{\footnotesize
\textbf{Geometric meaning of the HFS components.}
All quantities arise from the same multiscale heat field but probe
complementary aspects of point-cloud geometry.}
\label{tab:app_hfs_summary}
\setlength{\tabcolsep}{4pt}
\renewcommand{\arraystretch}{1.15}
\resizebox{.8\linewidth}{!}{%
\begin{tabular}{lll}
\toprule
Component & Level & Main geometric information \\
\midrule
$E_2(t)$
& Global
& Heat concentration / pairwise overlap \\
$C_2(t)$
& Global
& Scale response and codimension-sensitive behavior \\
$\mathcal E(t)$
& Global
& Roughness of the heat landscape \\
$d_{\mathrm{heat}}(x,t)$
& Local
& Effective intrinsic dimension \\
$2t\lambda_k(x,t)$
& Local
& Tangent--normal structure and anisotropy \\
$\tau(x,t)$
& Local
& Boundaries, junctions, and scale transitions \\
HDS
& Shape
& Distribution of dimensional types and density--geometry interaction \\
\bottomrule
\end{tabular}}
\end{table}

\section{Dataset Details}
\label{app:datasets}

Table~\ref{tab:datasets} summarizes the seven main benchmarks. The
synthetic datasets probe density, scale, and topology in $\mathbb R^2$,
while the real-world datasets span neuronal, subcellular, tree, and protein
morphology in $\mathbb R^3$.

\begin{table}[t]
\centering
\caption{\footnotesize
\textbf{Benchmark summary.} Synthetic benchmarks report overall accuracy (OA) over repeated train/test splits;
real-world benchmarks use stratified 5-fold CV and report mACC
(mean class accuracy).}
\label{tab:datasets}
\setlength{\tabcolsep}{5pt}
\renewcommand{\arraystretch}{1.12}
\resizebox{\linewidth}{!}{%
\begin{tabular}{lllcrrcl}
\toprule
\textbf{Type} & \textbf{Dataset} & \textbf{Domain} & \textbf{Classes} &
\textbf{Clouds} & \textbf{Points} & \textbf{Metric} & \textbf{Evaluation} \\
\midrule

\multirow{3}{*}{\textbf{Synthetic}}
& Processes
& Spatial processes
& 4 & 4,000 & 105--321
& OA
& $20\times$ 80/20 \\

& Orbit5k
& Dynamical systems
& 5 & 5,000 & 1,000
& OA
& $20\times$ 70/30 \\

& DisksAnnuli
& Density + topology
& 5 & 5,000 & 874--2,717
& OA
& $20\times$ 80/20 \\

\midrule

\multirow{4}{*}{\textbf{Real-world}}
& NeuroMorpho
& Neurons
& 9 & 6,198 & 1,024
& mACC
& 5-fold CV \\

& Allen
& Subcellular structures
& 7 & 8,400 & 1,024
& mACC
& 5-fold CV \\

& FOR-species20K
& LiDAR trees
& 15 & 3,435 & 1,024
& mACC
& 5-fold CV \\

& SCOP
& Protein folds
& 15 & 1,455 & 128
& mACC
& 5-fold CV \\

\bottomrule
\end{tabular}%
}
\end{table}
\subsection{Synthetic Benchmarks}
\label{app:datasets_synthetic}

\paragraph{Random Point Processes (Processes).}
We follow~\citet{kerber2024graphcode} and classify Poisson, Mat\'ern,
  Strauss, and Baddeley--Silverman processes, representing random,
  clustered, repulsive, and multiscale spatial interactions. Each class contains 1,000 clouds in $[0,1]^2$, generated by our
  re-implementation of the four processes with the parameters of
  \citet{kerber2024graphcode}. Because the processes are random, the
  point count varies between 105 and 321 (mean 194). We report mean
  accuracy over 20 random 80/20 splits.

\paragraph{Orbit5k.}
Orbit5k~\citep{carriere2020multiparameter} contains five classes generated
by the standard dynamical system with
$r\in\{2.5,3.5,4.0,4.1,4.3\}$, with 1,000 clouds of 1,000 points per
class. The hardest classes have similar coarse geometry but differ in
fine-scale density structure. We report mean accuracy over 20 random
70/30 splits following prior work.

\paragraph{Disks and Annuli (DisksAnnuli).}
Following the GraphCode benchmark of~\citet{kerber2024graphcode}, each
point cloud contains dense disks or annuli embedded in uniform background
noise, with object density $2.4$--$4.8\times$ the background density. The benchmark
jointly tests density-based localization and disk-versus-annulus topology.
Point counts are not fixed: as in the original construction, they depend on
the sampled areas and densities of the generated shapes and background,
yielding 874--2,717 points per cloud. We report mean accuracy over 20 random
80/20 splits.

\subsection{Real-World Morphology Benchmarks}
\label{app:datasets_realworld}

\paragraph{NeuroMorpho.}
This is a balanced nine-class benchmark from complete 3D neuronal
reconstructions in NeuroMorpho.Org~\citep{ascoli2007neuromorpho,
ascoli2018open}: basket, bipolar, ganglion, granule, medium spiny, mitral,
Purkinje, pyramidal, and stellate neurons. The resulting dataset contains
6,198 clouds. Each reconstruction is converted to a point cloud by
  sampling $m=1{,}024$ points uniformly by arc length along its segments;
  coordinates are kept in their native micron units. We use stratified
  5-fold cross-validation.

\paragraph{Allen.}
We use the \texttt{other\_punctate} subset of the 3D subcellular
point-cloud benchmark of~\citet{vasan2025interpretable}, derived from the
hiPSC imaging collection of~\citet{viana2023integrated}. The seven classes
correspond to CETN2, HIST1H2BJ, NUP153, RAB5A, SLC25A17, SMC1A, and SON.
Restricting to interphase cells and balancing the classes gives 8,400
clouds. Each is randomly subsampled to $m=1{,}024$ points. The task primarily distinguishes multiscale spatial
distributions and clustering of intracellular structures. We use
stratified 5-fold cross-validation.

\paragraph{FOR-species20K.}
FOR-species20K~\citep{puliti2025forspecies20k} contains terrestrial
laser-scanning point clouds of individual trees. We use the 15 most
populated species and balance them to 229 trees each, yielding 3,435
clouds. Each tree is randomly subsampled to $m=1{,}024$ points. The
  strongly nonuniform LiDAR density caused
by occlusion and viewing geometry provides a useful robustness test.
We use stratified 5-fold cross-validation.

\paragraph{SCOP.}
We use the SCOP structural-classification task from
ProteinShake~\citep{proteinshake}. Following standard residue-level
point-cloud representations of protein structure, each cloud contains only the
$C_\alpha$ backbone coordinates; no sequence, residue identity, side-chain, or
chemical features are provided. We select the 15 most populated folds and
balance them to 97 proteins each, yielding 1,455 clouds.

Protein structures have variable native resolution, with approximately one
$C_\alpha$ point per residue. We initially used a point budget of
$m=1{,}024$, but upsampling shorter proteins created coincident points. These
zero-distance duplicates caused the nearest-neighbor scale
$r_{\mathrm{nn}}$, and hence $t_{\min}$, to collapse toward zero, corrupting
the finest-scale heat features. We therefore use $m=128$, subsampling longer
proteins as needed and avoiding artificial density structure from duplicated
residues. Clouds are centered but not orientation-normalized, making SCOP a
natural test of rotation behavior. We use stratified 5-fold cross-validation.
Because all methods receive coordinates only, these results should not be
compared directly with protein models that additionally use sequence, residue
identity, or other chemical information.

\section{Experimental Details}
\label{app:experimental_details}

\subsection{Implementation Details}
\label{app:implementation}

\paragraph{Heat field signatures.}
HFS-simple, HFS-desc and HFS-full share one scale schedule: $T=8$ diffusion times spaced
logarithmically between $c_{\min}\,r_{\mathrm{nn}}^2$ and $c_{\max}\,\mathrm{diam}^2$
with $c_{\min}=0.05$, $c_{\max}=0.25$ (Appendix~\ref{app:scale_choices}); no
per-dataset tuning is performed. Here $r_{\mathrm{nn}}$ is the mean
nearest-neighbour distance, computed over \emph{distinct} points: a point cloud
may contain coincident samples (for instance when a short protein is resampled to
a fixed budget), and a duplicate lying on top of its own copy is not a neighbour.
Measuring $r_{\mathrm{nn}}$ over distinct points keeps it a proxy for the sampling
resolution and keeps the schedule well defined for such clouds. This yields
  $6T=48$ features per point in $\mathbb{R}^3$. Each local channel ($u_X$,
  $d_{\mathrm{heat}}$, $\tau$, and the three eigenvalues $2t\lambda_k$) is
  pooled over points at every scale by its mean, standard deviation and
  10th/50th/90th percentiles ($5T=40$ values per channel); concatenating the
  global ($32$), HDS ($64$) and pooled-local ($240$) blocks gives the
  $336$-dimensional HFS-desc descriptor. \textbf{HFS-simple} omits the log-Hessian block
and is $216$-dimensional. \textbf{HFS-full} keeps the same per-point features
and replaces the fixed statistical pooling with a learned shared MLP of
  widths $(64,128,256)$ with batch normalization, mean and max pooling, and a
  $256$-unit head with dropout $0.5$ ($\sim\!0.2$M parameters), trained for $250$
epochs with Adam at learning rate $10^{-3}$ under a cosine schedule, batch size
$32$; its three input blocks are standardized per feature using training-split
  statistics. \textbf{HFS+DGCNN} concatenates the standardized $48$ per-point HFS
  features with $xyz$ and feeds the result to a DGCNN whose first $k$-NN graph is
  built on $xyz$ only and which is otherwise unmodified.

\paragraph{Classical descriptors.}
HKS is computed on a symmetric $10$-NN graph Laplacian ($80$ eigenpairs,
    $16$ log-spaced times), and FPFH with Open3D on unit-ball-normalized clouds
    (normal radius $0.05$, feature radius $0.10$); both are pooled over points
    into a fixed-length cloud descriptor. On SCOP, where $m=128$ makes these
    radii enclose no neighbours for roughly a third of the clouds, FPFH falls
    back to $30$- and $100$-nearest-neighbour neighbourhoods for those clouds.

\paragraph{Classifier for descriptor rows.}
Every closed-form descriptor row, FPFH, HKS, HFS-simple, and HFS-desc, is
standardized and classified by the \emph{same} two-layer MLP with hidden widths
$(128,64)$, trained to convergence (at most $3000$ iterations; $1000$ for the
  synthetic HFS rows and the scale ablation). Holding the
classifier fixed across these rows is deliberate: it makes the comparison a
comparison of \emph{representations}, not of downstream models.

\paragraph{Deep baselines.}
PointNet, PointNet++, DGCNN, and Point Transformer use their authors' published
optimizers, schedules, and epoch budgets (e.g. DGCNN: SGD with momentum $0.9$,
weight decay $10^{-4}$, cosine schedule, label smoothing $0.2$, $250$ epochs;
Point Transformer: SGD, multistep schedule, $200$ epochs), with batch size $32$.
Training augmentation is the standard recipe of that literature: random rotation
about the $z$ axis, random scaling in $[0.8,1.25]$, random translation, and
per-point jitter. We deliberately do not add full $SO(3)$ augmentation to these
four, since doing so would depart from the published protocols that produce their
reported numbers; Section~\ref{sec:results} analyzes the consequence and adds
two baselines that address rotation directly. \textbf{DGCNN-SO(3)} is the same
DGCNN architecture under the same recipe, differing \emph{only} in that its
rotation augmentation is drawn from the full Haar measure on $SO(3)$ rather than
about the $z$ axis alone, so any difference is attributable to the augmentation.
\textbf{VN-DGCNN} replaces DGCNN's layers with Vector
Neurons~\citep{deng2021vector}, which carry lists of $3$-vectors instead of
scalars and are therefore rotation-equivariant by construction, with an invariant
readout; it uses the same optimizer, schedule and epoch budget as DGCNN. We
verified its invariance numerically rather than assuming it: rotating an input
changes its logits by $6\times 10^{-8}$ relative, i.e. float32 round-off.

\paragraph{Point budget.}
On the four real benchmarks every method receives the identical clouds, so the
comparison there isolates the representation exactly. Deep baselines center each
  cloud and scale it to the unit ball inside their data loaders, following their
  published recipes; the HFS features are computed on the raw coordinates, since
  the HFS scale schedule adapts to each cloud. On the synthetic benchmarks
the descriptor methods, which accept variable-size input, use the clouds at native
resolution while the deep baselines resample to a fixed budget ($m=256$ for
  Processes, $m=1{,}024$ for Orbit5k and DisksAnnuli), so the two families
see the same clouds at different resolutions. Clouds are resampled to a fixed
budget $m$ per dataset: $m=1{,}024$
for the subcellular, neuron and tree benchmarks, whose raw acquisitions contain far
more points than that (Allen images, for example, provide $20{,}480$), so the
budget is a genuine \emph{sub}sample. Proteins are the exception: a protein has a
native resolution of one point per residue (median $189$ in our subset), so a
budget of $1{,}024$ would not subsample but \emph{duplicate} real residues roughly
five-fold, inventing structure the molecule does not have and injecting spurious
density variation into a density-sensitive descriptor. We therefore use $m=128$ for
SCOP, at which the median protein is represented at native resolution with no
duplicated residue and long chains are subsampled. We note the consequence
honestly: PointNet++ is the one baseline whose accuracy falls at the lower budget,
because it alone builds its hierarchy by repeatedly subsampling the input
($m\!\to\!m/2\!\to\!m/8$ centroids), whereas the $k$-NN-based baselines and the
descriptor rows improve once the duplication is removed.

\paragraph{Protocol and hardware.}
Synthetic benchmarks use $20$ random stratified splits ($80/20$; Orbit5k $70/30$),
with the split seeds shared by every method; real benchmarks use stratified
$5$-fold cross-validation, with the fold assignment fixed once and shared by every
method, so all methods are trained and evaluated on exactly the same partitions.
Degenerate clouds, a small number of real scans are too sparse or nearly
collinear for the Hessian eigendecomposition, are detected per cloud and
imputed rather than dropped, so the sample count stays identical across methods;
imputation is applied per feature block, so a failure in one channel does not
discard the channels that were computed successfully.
All experiments run on one node with a single NVIDIA H200 GPU and $16$ CPU
workers; the timings in Table~\ref{tab:time} were collected on that same node for
every method, with the closed-form descriptors extracted on the CPU worker pool
(they use no GPU) and the deep networks trained end-to-end on the GPU.

\subsection{Boundary Datasets: Where HFS Is Neutral}
\label{app:boundary}

HFS is designed for labels carried by intrinsic multiscale geometry. It is worth
stating plainly where that assumption does \emph{not} hold, both because it
sharpens the claim and because the boundary is informative.

We report two such datasets. \textbf{IntrA}~\citep{yang2020intra} asks whether a
reconstructed vessel segment contains an intracranial aneurysm, a single
localized bulge on an otherwise smooth tube, so the label is a property of one
region of the surface rather than of the multiscale organization of the whole
cloud. \textbf{ATLAS-1}~\citep{devries2025pointmil} classifies drug treatment
from the reconstructed surface of a single cell, where the discriminative signal
is largely overall cell shape and size.

Table~\ref{tab:boundary} shows the expected pattern: coordinate-based networks are
at least as strong as HFS on both, and the margins that HFS enjoys on the
morphology benchmarks disappear. HFS remains competitive rather than
uncompetitive. On IntrA the best HFS variant is second by mACC and $2.8$ points
  behind DGCNN; on ATLAS-1 it is $1.7$ points behind Point Transformer. But
neither is a task where HFS is the right tool, and we do not claim it is.

\begin{wraptable}{r}{3in}
\vspace{-.2in}
\centering
\caption{\footnotesize \textbf{Boundary datasets} (mACC / macro-F1, \%). Both
labels are carried by rigid or localized shape rather than by intrinsic
multiscale geometry, and HFS is neutral there. IntrA uses the official split;
ATLAS-1 follows the protocol of~\citet{devries2025pointmil}; DGCNN and
  HFS+DGCNN are averaged over three seeds per fold. \textbf{Bold} =
best, \underline{underline} = second.}
\label{tab:boundary}
\begin{tabular}{l cc}
\toprule
\textbf{Method} & IntrA & ATLAS-1 \\
\midrule
FPFH         & 81.6 / 82.1 & 74.6 / 74.5 \\
HKS       & 62.6 / 64.1 & 50.3 / 50.3 \\
\midrule
PointNet    & 71.5 / 73.6 & 74.9 / 74.3 \\
PointNet++ & 89.0 / \underline{90.9} & 76.1 / 76.1 \\
DGCNN     & \textbf{92.6} / \textbf{91.9} & 76.3 / 76.2 \\
Point Tr     & 79.0 / 80.4 & \textbf{78.1} / \textbf{78.0} \\
\midrule
\rowcolor{green!8}
HFS-simple                         & 80.1 / 81.0 & 73.0 / 72.9 \\
\rowcolor{green!8}
HFS-desc                           & 85.1 / 86.0 & 74.4 / 74.3 \\
\rowcolor{green!8}
HFS-full                           & 89.2 / 89.0 & 74.4 / 74.4 \\
\rowcolor{green!8}
HFS+DGCNN                          & \underline{89.8} / 90.8 & \underline{76.4} / \underline{76.6} \\
\bottomrule
\end{tabular}
\vspace{-.3in}
\end{wraptable}
We read this as support for the paper's thesis rather than against it. A
representation that computes density, intrinsic dimension and scale transitions
should help exactly when those quantities carry the label, and should be neutral
when they do not; a method that won everywhere would be evidence that the
benchmarks, not the representation, were doing the work. The practical guidance
follows directly: use HFS where the class is intrinsic morphology, and a
coordinate encoder where it is rigid shape or localized appearance.

\subsection{Scale-Schedule Ablation}
\label{app:ablations}

The scale schedule has three constants: the number of diffusion times $T$ and the
range multipliers $c_{\min},c_{\max}$. Table~\ref{tab:scale_ablation} varies them
on one synthetic and one real dataset, holding everything else fixed.

Accuracy increases monotonically with $T$ and saturates: on NeuroMorpho
$58.8\rightarrow69.5\rightarrow72.3\rightarrow74.2$ for $T=2,4,8,16$, with the
step from $8$ to $16$ worth $1.9$ points for twice the descriptor length. This is
the expected signature of a genuinely multiscale representation, a single
scale is not enough, and the marginal value of finer sampling decays. We use
$T=8$ throughout as the accuracy/size compromise, not as a tuned optimum: $T=16$
would improve every reported HFS number.

\begin{wraptable}{r}{3in}
\vspace{-.2in}
\centering
\caption{\footnotesize \textbf{Scale-schedule ablation.} Top: number of scales
$T$ (range fixed at $c_{\min}=0.05$, $c_{\max}=0.25$). Bottom: range multipliers
at $T=8$. Processes: OA over 20 splits; NeuroMorpho: mACC over 5 folds. The
defaults used throughout are $T=8$, $(0.05,0.25)$ (shaded).}
\label{tab:scale_ablation}
\begin{tabular}{l r cc}
\toprule
\textbf{Setting} & dim & Processes & NeuroMorpho \\
\midrule
 $T=2$  &  84 & 95.0\std{0.8} & 58.8\std{1.5} \\
  $T=4$  & 168 & 97.7\std{0.5} & 69.5\std{1.6} \\
  \rowcolor{green!8}
  $T=8$  & 336 & 98.3\std{0.4} & 72.3\std{1.5} \\
  $T=16$ & 672 & 98.6\std{0.4} & 74.2\std{1.0} \\
  \midrule
  $(c_{\min},c_{\max})$ & & & \\
  $(0.01,0.25)$ & 336 & 98.6\std{0.3} & 72.2\std{0.8} \\
  $(0.05,0.10)$ & 336 & 98.3\std{0.3} & 72.4\std{1.5} \\
  \rowcolor{green!8}
  $(0.05,0.25)$ & 336 & 98.3\std{0.4} & 72.3\std{1.5} \\
  $(0.05,0.50)$ & 336 & 98.2\std{0.5} & 72.7\std{1.2} \\
  $(0.10,0.25)$ & 336 & 98.5\std{0.4} & 73.0\std{1.7} \\

\bottomrule
\end{tabular}
\vspace{-.3in}
\end{wraptable}
Sensitivity to the range is small. Across the five $(c_{\min},c_{\max})$ settings
accuracy spans $0.4$ points on Processes
  ($98.2$--$98.6$) and $0.8$ on NeuroMorpho ($72.2$--$73.0$), comparable to the fold-to-fold standard deviation.
The defaults $(0.05,0.25)$ are therefore not a tuned choice, and no per-dataset
tuning is performed anywhere in the paper.

\subsection{Rotation-Invariance Protocol}
\label{app:rotation}

\noindent \textbf{Protocol.} \quad
Each method is trained \emph{once}, on the unmodified training split, and then
evaluated twice: on the original test clouds, and on the \emph{same} test clouds
after each has been rotated by an independently drawn random rotation. No method is
  trained on rotated test clouds; the only rotations seen in training are the
  augmentation of each recipe (rotation about $z$ for every network except
  DGCNN-SO(3), which uses full $SO(3)$), and no model is retrained between the two
evaluations, so the reported quantity isolates the sensitivity of the learned or
computed representation to a change of frame. The reported $\Delta$ is the
balanced accuracy on the original test set minus the balanced accuracy on the
rotated test set, so $0$ means the rotation had no effect and positive values are
accuracy lost. Because a single fixed split is used for both evaluations, values
within roughly $\pm 2$ points are run-to-run noise; the small negative entries
in Table~\ref{tab:rotation} are gains of this size and should be read as zero.

\noindent \textbf{Sampling rotations.} \quad
Rotations are drawn uniformly from $SO(3)$ under the Haar measure. We sample a
matrix with i.i.d. standard normal entries, take its QR decomposition, correct
the column signs by the signs of the diagonal of $R$, and negate one column if
the determinant is $-1$; this yields a uniformly distributed element of $SO(3)$
rather than the non-uniform distribution obtained from uniformly sampled Euler
angles. Each test cloud receives its own independent rotation.

\noindent \textbf{Feature-level check.} \quad
Accuracy alone cannot distinguish a representation that is genuinely invariant
from one that merely happens to be robust on a particular dataset, so we also
measure invariance directly on the features. For each test cloud we compute the
relative deviation $\lVert \Phi(X) - \Phi(RX)\rVert_2 \,/\, \lVert
\Phi(X)\rVert_2$ between the descriptor of the original cloud and of its rotated
copy, and average over the test set. For all HFS variants this quantity is at
the level of floating-point round-off ($10^{-7}$ to $4\times10^{-6}$), confirming that the zero
entries in Table~\ref{tab:rotation} reflect exact invariance by construction
rather than an empirical coincidence. This is the distinction the
\emph{Invariance} column reports.

\noindent \textbf{Measuring canonical orientation.} \quad
The claim that SCOP lacks a canonical frame, while trees and cells have one, is
quantified as follows. For each cloud we take the principal axis $v(X)$, the unit
eigenvector of the covariance matrix with the largest eigenvalue, resolving its
sign ambiguity by fixing a hemisphere. If a dataset has a canonical orientation,
these axes concentrate in one direction and the norm of their mean,
$\lVert \frac{1}{N}\sum_i v(X_i) \rVert_2$, is close to $1$; if orientations are
arbitrary, the axes spread over the sphere and the mean norm approaches $0$. This
alignment score is $0.94$ for FOR-species, where gravity fixes the trunk axis,
and $0.44$ for SCOP, whose PDB reference frames are arbitrary, which is why
rotating the SCOP test set is nearly in-distribution for the baselines and costs
them little.

\section{Multiparameter Persistence for Point Clouds}
\label{app:mp}

This appendix gives a brief, self-contained account of multiparameter
persistence and of the density--distance bifiltration used by the topological
baselines in Section~\ref{sec:results}. We assume familiarity with homology but
not with persistence.

\subsection{One-parameter persistence}
\label{app:mp_single}

Persistent homology summarizes how the topology of a point cloud evolves across
a single scale. Given $X$, one builds a growing family of simplicial
complexes (a \emph{filtration}) indexed by a scale $r\ge 0$, most commonly
the Vietoris--Rips complex
\[
    \mathrm{Rips}_r(X)
    =
    \bigl\{\, \sigma \subseteq X :
    \|x_i - x_j\| \le r \ \text{for all } x_i,x_j \in \sigma \,\bigr\}.
\]
As $r$ increases, simplices are only added, so $\mathrm{Rips}_r(X)\subseteq
\mathrm{Rips}_{r'}(X)$ for $r\le r'$. Applying $k$-th homology yields a sequence
of vector spaces linked by inclusion-induced maps, a \emph{persistence
module}, whose structure is captured completely by a \emph{barcode}: a
multiset of intervals recording the scale at which each topological feature
(component, loop, void) is born and dies. Barcodes are stable to perturbations
of $X$~\citep{edelsbrunner2002topological} and can be vectorized for machine
learning.

\paragraph{Why one parameter is not enough.}
A single scale conflates density with geometry. In a cloud with nonuniform
sampling, the scale needed to connect a sparse filament already collapses a
dense cluster, and outliers create short-lived features indistinguishable from
genuine sparse structure. Distinguishing a dense annulus from background noise,
or a repulsive point process from a clustered one, requires reasoning about
topology \emph{and} density jointly, which a one-parameter filtration cannot
express.

\subsection{Multiparameter persistence}
\label{app:mp_multi}

Multiparameter persistence addresses this by filtering along two (or more)
parameters at once. One indexes complexes by a pair $(s,r)\in\mathbb{R}^2$ that
is monotone in each coordinate, obtaining a \emph{bifiltration}
$\{K_{s,r}\}$ with $K_{s,r}\subseteq K_{s',r'}$ whenever $s\le s'$ and $r\le
r'$. Homology now produces a \emph{two-parameter persistence module}: a grid of
vector spaces $M_{s,r}=H_k(K_{s,r})$ with commuting maps induced by the
inclusions~\citep{botnan2022introduction,coskunuzer2024topological}.

The essential difficulty is that, unlike the one-parameter case, two-parameter
modules admit \emph{no} complete discrete invariant analogous to the
barcode~\citep{carlsson2009theory}. There is no finite list of intervals that
recovers the module up to isomorphism. Practical pipelines therefore replace
the (intractable) full module with stable \emph{feature maps} that send the
module to a vector or graph usable by a classifier. The baselines in
Section~\ref{sec:results} are instances of this: multiparameter persistence
images and landscapes, GRIL, and Hilbert signed-measure convolutions each
vectorize the module for XGBoost, while graphcodes embed it as a graph for a
GNN~\citep{carriere2020multiparameter,vipond2020multiparameter,xin2023gril,loiseaux2023stable,kerber2024graphcode}.

\subsection{The density--distance bifiltration}
\label{app:mp_bifiltration}

For point clouds, the canonical choice of two parameters is \emph{distance}
(geometric scale) and \emph{density}. Fix a density estimate; a common choice
is a Gaussian kernel density estimator with bandwidth $\delta>0$,
\[
    \rho_\delta(x)
    =
    \frac{1}{m}\sum_{j=1}^{m}
    \exp\!\left(-\frac{\|x-x_j\|^2}{2\delta^2}\right),
\]
and let $\gamma(x)=-\rho_\delta(x)$ be the corresponding \emph{codensity}, so
that low codensity marks dense regions. The density--distance bifiltration is
then built in two nested steps. For a density threshold $s$ we first restrict to
the sufficiently dense subsample
\quad $    X_s = \{\, x\in X : \gamma(x)\le s \,\},$\quad
and then build the Rips complex at scale $r$ on that subsample,
\quad $    K_{s,r} = \mathrm{Rips}_r(X_s).$\quad 
Both operations are monotone, raising $s$ admits sparser points, raising $r$
adds simplices, so $\{K_{s,r}\}_{(s,r)}$ is a genuine bifiltration and its
$k$-th homology is a two-parameter module. Intuitively, the density axis peels
the cloud from its densest core outward, while the distance axis grows
connectivity at each density level; a topological feature is characterized by
the \emph{region} of the $(s,r)$-plane over which it persists, not by a single
birth--death pair. Related density-aware constructions replace the KDE
sublevel filtration with a distance-to-measure or multicover
filtration~\citep{edelsbrunner2021multi}, but the two-parameter density/scale
structure is the same.

\paragraph{Why this matches the benchmarks.}
This construction is exactly what the Processes and DisksAnnuli benchmarks are
designed to probe. On DisksAnnuli, an annulus contributes an $H_1$ loop that is
present only within the dense-object region (a band on the density axis) and
only across a range of scales (a band on the distance axis); a disk contributes
none, and uniform background noise contributes neither, so the classes
separate jointly in $(s,r)$ though not in either parameter alone. On Processes,
attractive, repulsive, and multiscale interactions imprint distinct joint
density--scale signatures. HFS targets this same density--geometry entanglement,
but through closed-form differential channels of the heat field
(Section~\ref{sec:method}) rather than by computing and vectorizing a
persistence module.

\end{document}